%% file: acl_latex.tex
\documentclass[11pt]{article}

\usepackage[preprint]{acl}

\usepackage{times}
\usepackage{latexsym}

\usepackage[T1]{fontenc}

\usepackage[utf8]{inputenc}

\usepackage{microtype}

\usepackage{inconsolata}

\usepackage{graphicx}

\usepackage{amsmath}
\usepackage{xspace}
\newcommand{\ours}{GAZER\xspace}
\newcommand{\stageone}{Reflective Diagnosis\xspace}
\newcommand{\stagetwo}{Semantic Correction\xspace}
\newcommand{\ar}{AVM\xspace}
\newcommand{\ars}{AVMs\xspace}

\usepackage[textwidth=3.5cm]{todonotes}

\usepackage{amssymb}
\usepackage{algorithm}
\usepackage{algorithmic}

\usepackage{fontawesome}
\usepackage{multirow} 
\usepackage{colortbl} 
\usepackage{booktabs}
\definecolor{darkgreen}{RGB}{0, 100, 0}
\usepackage{siunitx}
\usepackage{makecell}
\usepackage{tcolorbox}
\tcbuselibrary{listings,breakable}
\usepackage{amsthm}
\definecolor{mypink}{RGB}{255,231,226}

\newtheorem{theorem}{Theorem}[section]

\newtheorem{assumption}[theorem]{Assumption}
\newtheorem{remark}[theorem]{Remark}



\usepackage{soul}
\usepackage{xcolor} 
\usepackage{colortbl}
\usepackage{tcolorbox}
\newtcolorbox{mybox3}[1]{colbacktitle=white,coltitle=black,colback=white,colframe=black,fonttitle=\bfseries,fontupper=\small,title=#1,leftupper=0.5em,rightupper=0.5em,boxrule=1.0pt}

\usepackage{enumitem}

\tcbset{
    fontupper=\ttfamily\small, 
    fontlower=\ttfamily\small, 
    halign=left,    
    boxsep=1mm, left=1mm, right=1mm, top=1mm, bottom=1mm
}

\title{Training-Free Semantic Correction for Autoregressive Visual Models}

\author{
  \textbf{Junhao Chen\textsuperscript{1}},
  \textbf{Chanyu Zhu\textsuperscript{2}},
  \textbf{Zheqi Lv\textsuperscript{1}},
  \textbf{Keting Yin\textsuperscript{1}},
  \textbf{Shengyu Zhang\textsuperscript{1}}
\\
\\
  \textsuperscript{1}Zhejiang University,
  \textsuperscript{2}Shandong University
\\
  \small{
    \texttt{\{chenjunhao100, zheqilv, yinkt, sy\_zhang\}@zju.edu.cn},
    \texttt{zhuchanyu@mail.sdu.edu.cn}
  }
}

\begin{document}
\maketitle

\begin{abstract}
\input{0_abstract}

\end{abstract}

\section{Introduction}
\label{introduction}
\input{1_introduction}

\section{Related Work}
\label{related_work}
\input{2_related_work}

\section{Methodology}
\label{methods}
\input{3_methods}

\section{Experiments}
\label{experiments}
\input{4_experiments}

\section{Conclusion}
\label{conclusion}
\input{6_conclusion}

\section*{Limitations}

Although our experiments demonstrate the effectiveness of \ours across text-to-image and text-to-video generation, several aspects merit further study. \ours is most suitable for errors that can be identified from intermediate semantic evidence, such as objects, attributes, relations, and visible events. More ambiguous prompts or very fine-grained temporal requirements may benefit from stronger feedback models and more specialized verification signals. Furthermore, \ours is designed for next-scale prediction models and does not directly transfer to next-token prediction backbones, where raster-scan intermediate states lack the globally coherent visual structure required for reliable rollout preview diagnosis (see Appendix~\ref{appendix:failure} for details). Future work will explore these directions to make training-free semantic correction more efficient and broadly applicable.

\section*{Ethics Statement}

This work is intended solely for research on improving semantic alignment in text-to-image and text-to-video generation. Our experiments use public benchmarks and publicly available or open-source generation backbones and evaluation models, and we follow the licenses and usage policies associated with these resources. The generated examples are used only for analysis and illustration of model behavior. Since stronger alignment and correction methods could also be misused to create more convincing synthetic content, we encourage responsible deployment with appropriate safeguards, including prompt filtering, provenance tracking, and human review in sensitive applications. We do not use private user data or personally identifying information in our experiments. We used AI assistants for minor writing assistance in preparing this manuscript. All technical content, experiments, and conclusions are the work of the authors.



\bibliography{custom}

\clearpage
\appendix

\input{7_appendix}

\end{document}

%% file: 0_abstract.tex
Autoregressive visual models (\ars) based on next-scale prediction have emerged as a prominent paradigm for image and video synthesis. However, decomposing the generation process into discrete scales with varying granularities in \ar makes semantic errors difficult to identify and correct, thereby undermining the quality of the final output. Prior efforts to enhance \ar can be categorized into training-based and training-free approaches. 
Although training-based efforts to enhance \ar generation quality come at substantial computational cost, existing training-free methods neglect intermediate generation states, leaving semantic errors undiagnosed and allowing them to accumulate into the final output.
In this paper, we focus on training-free paradigms and propose \ours, a framework that integrates multimodal large language model feedback into the \ar sampling loop for in-generation semantic correction. Concretely, \ours operates via two cooperating stages: the \stageone stage diagnoses semantic errors from intermediate states, while the \stagetwo stage rewinds and rectifies the generation trajectory to realign with the target prompt. Experiments on compositional image and video benchmarks demonstrate that \ours improves semantic alignment and compositional accuracy across multiple \ars without additional training.\footnote{Code is available at \url{https://github.com/June-Hall/Gazer}.}

%% file: 1_introduction.tex
\begin{figure*}[t]
    \centering
    \includegraphics[width=\textwidth]{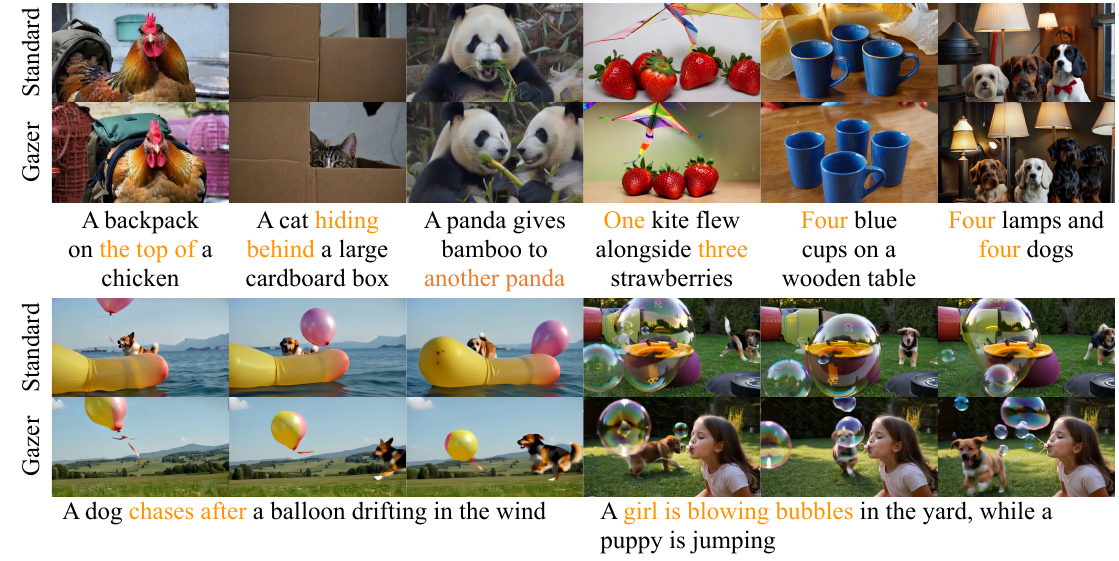}
    \caption{
    \textbf{Qualitative comparison across text-to-image (T2I) and text-to-video (T2V) tasks.}
    Each block corresponds to a text prompt.
    We show results from the baseline and our method for both image generation (top)
    and video generation (bottom, visualized as frame sequences).
    Our approach achieves better semantic alignment with the prompt
    and more coherent temporal dynamics.
    }
    \label{fig:qualitative}
\end{figure*}

Autoregressive (AR) visual models (\ars) have become a leading paradigm for image and video synthesis~\citep{ramesh2021dalle, chen2020generative, lee2022rqtransformer, yu2022scaling}. 
Following a coarse-to-fine design, recent models such as VAR~\citep{tian2024var}, Infinity~\citep{han2025infinity}, InfinityStar~\citep{liu2025infinitystar}, and STAR~\citep{ma2025star} have achieved strong results in image and video generation by structuring generation as hierarchical next-scale prediction across progressively refined semantic scales.
However, decomposing generation into discrete scales with varying granularities introduces a fundamental tension that compromises semantic alignment between the generated output and the input prompt. Such degradation stems from two compounding challenges: semantic errors within intermediate states are difficult to diagnose, and the unidirectional nature of \ar generation precludes revisitation or correction once such errors arise, allowing them to propagate across scales and degrade the quality of the final output.

Prior efforts to improve semantic alignment in \ar can be divided into two categories: 
(i) \textbf{Training-based alignment.} Frameworks such as T2I-R1~\citep{jiang2025t2ir1} fine-tune the generator with reinforcement learning to reshape its output distribution toward prompt semantics, yielding measurable improvements in semantic alignment. However, they require large-scale training on massive high-quality data, demanding substantial computational resources and limiting their applicability in the real world.
(ii) \textbf{Training-free alignment.}
This category includes prompt optimization before generation and output refinement after generation. 
For instance, Promptist~\citep{hao2023promptist} and OPT2I~\citep{manas2024opt2i} refine the input prompt to provide a stronger initial condition, while PARM and PARM++~\citep{guo2025parm} score generated candidates and iteratively regenerate them until alignment improves. 
However, existing training-free alignment methods operate only at the beginning or the end of the generation process, neglecting intermediate generation states and leaving semantic errors undiagnosed until they accumulate into the final output.

To address this gap, we focus on training-free paradigms and propose \ours, a framework that integrates multi-modal large language model (MLLM) feedback into the \ar sampling loop for in-generation semantic correction. 
Intuitively, the \ar, like a sketching artist, should gaze at its evolving draft and revise it before adding finer detail. 
Concretely, at selected intermediate scales, GAZER lets the model gaze at its own generation through an MLLM, extracting a diagnostic signal that drives the next sampling step toward the target semantics. 
As a result, GAZER transforms a passive sampler into a generator that revises its own draft via MLLM feedback, improving semantic alignment at inference time.

As illustrated in Figure~\ref{fig:method}, GAZER consists of two cooperating stages: 
(i) \textbf{\stageone}. Intermediate generation states cannot be reliably diagnosed, since they do not yet carry the coherent visual semantics required for accurate diagnosis. To this end, the \stageone stage constructs rollout previews at selected intermediate scales, transforming incomplete states into semantically readable visual predictions that enable the \ar to reflect on its evolving generation via an MLLM. Based on these previews, the MLLM produces semantic evaluations and corrective cues for subsequent refinement.
(ii) \textbf{\stagetwo}. With these corrective cues from \stageone, the remaining challenge is how to incorporate semantic correction into a unidirectional \ar trajectory without retraining the generator or breaking the structure already produced. To address this, the \stagetwo stage rewinds the generation to the previous scale and rectifies the generation trajectory by resampling from the rewound scale under the enhancement and suppression cues, redirecting the trajectory toward the target prompt. Together, the two stages enable in-generation semantic correction for existing \ar.

Our contributions are summarized as follows:
\begin{itemize}[leftmargin=1.6em, topsep=0pt, parsep=0pt, partopsep=0pt, itemsep=0.35em]%
    \item We propose \ours, a training-free framework that mitigates the accumulation and propagation of semantic errors across scales in \ar by introducing MLLM-driven diagnosis and correction during sampling.
    \item We design two stages for in-generation semantic correction in \ar: \stageone constructs rollout previews from intermediate states to enable MLLM feedback, and \stagetwo turns that diagnosis into a trajectory-level correction without retraining.
    \item Experiments on image and video compositional benchmarks show that \ours improves semantic alignment and compositional accuracy across multiple \ars.
\end{itemize}

%% file: 2_related_work.tex
\noindent \textbf{Autoregressive Visual Generation.}
Autoregressive (AR) visual generation spans several prediction paradigms~\citep{yan2021videogpt, wu2021nuwavisualsynthesispretraining, villegas2022phenakivariablelengthvideo, kondratyuk2024videopoetlargelanguagemodel, wu2022nuwainfinityautoregressiveautoregressivegeneration}, including next-token (raster-scan) prediction (Parti~\citep{yu2022parti}, LlamaGen~\citep{sun2024llamagen}), masked parallel decoding (MaskGIT~\citep{chang2022maskgit}, MAGE~\citep{li2023mage}), continuous-token AR (MAR~\citep{li2024mar}), and next-scale prediction (VAR~\citep{tian2024var}).
The next-scale family, extended by Infinity~\citep{han2025infinity}, InfinityStar~\citep{liu2025infinitystar}, HART~\citep{tang2024hart}, and STAR~\citep{ma2025star}, predicts visual tokens scale by scale and refines a low-resolution layout into fine-grained detail.
A shared property of this family is that sampling proceeds strictly from coarse to fine in one direction, so semantic decisions made at early scales are carried into all subsequent scales without revisiting.
\ours targets the next-scale AR family, introducing training-free in-generation semantic correction directly into the sampling process.


\noindent \textbf{Semantic Alignment in Visual Generation.} 
Semantic alignment has been explored in diffusion~\citep{chefer2023attendexcite, feng2023structurediffusion, rassin2023syngen, lv2025ppad}, language generation~\citep{madaan2023selfrefine, shinn2023reflexion}, and visual evaluation~\citep{kirstain2023pickscore, lin2024vqascore}, but semantic alignment in \ar has received comparatively less attention.
Efforts to improve semantic alignment in \ar fall into training-based and training-free approaches.

\noindent \textit{Training-based alignment.}
Training-based methods improve alignment by reshaping the output distribution of \ar through fine-tuning or reward optimization.
T2I-R1~\citep{jiang2025t2ir1} fine-tunes the generator with RL on compositional rewards, Diffusion-DPO~\citep{wallace2024diffusiondpo} and DDPO~\citep{black2024ddpo} adapt preference optimization to diffusion, and ImageReward~\citep{xu2023imagereward} and HPS~\citep{wu2023hpsv2} provide reward models from human preferences. 
However, these methods demand substantial computational resources and large-scale preference data, limiting their practical applicability.

\noindent \textit{Training-free alignment.}
Training-free methods avoid retraining and operate entirely at inference time, but differ in when the alignment signal is applied.
(i) Pre-generation methods condition on a refined input before sampling begins. 
Prompt-side methods such as Promptist~\citep{hao2023promptist}, an RL-trained rewriter, and OPT2I~\citep{manas2024opt2i}, which iteratively refines the prompt with an LLM, supply a stronger conditioning signal before sampling.
However, they cannot detect or respond to semantic drift that emerges during sampling.
(ii) Post-generation methods instead evaluate or refine completed outputs after sampling.
After sampling, PARM and PARM++~\citep{guo2025parm} score completed candidates and trigger regeneration, and ReNO~\citep{eyring2024reno} optimizes the initial noise to maximize a reward.
However, correction operates outside the generation trajectory and cannot intervene where errors first arise.

\ours requires no parameter updates or preference data, distinguishing it from training-based alignment methods. Unlike prompt optimization and output refinement approaches that operate outside the generation trajectory, \ours introduces semantic correction directly during sampling, addressing errors before they propagate across scales.

%% file: 3_methods.tex
As illustrated in Figure~\ref{fig:method}, \ours applies a two-stage correction, namely \textbf{\stageone} (§\ref{sec:diagnosis}) and \textbf{\stagetwo} (§\ref{sec:correction}), at a small set of intermediate scales, without modifying any model components.
\stageone reflects on the current trajectory by constructing rollout previews, enabling semantic evaluation via MLLM consultation, and producing enhancement and suppression cues for subsequent correction.
\stagetwo rewinds the trajectory to an earlier scale and rectifies subsequent generation under enhancement and suppression cues, redirecting the trajectory toward improved semantic alignment before resuming standard scale-by-scale sampling.

\begin{figure*}[t]
    \centering
    \includegraphics[width=\textwidth]{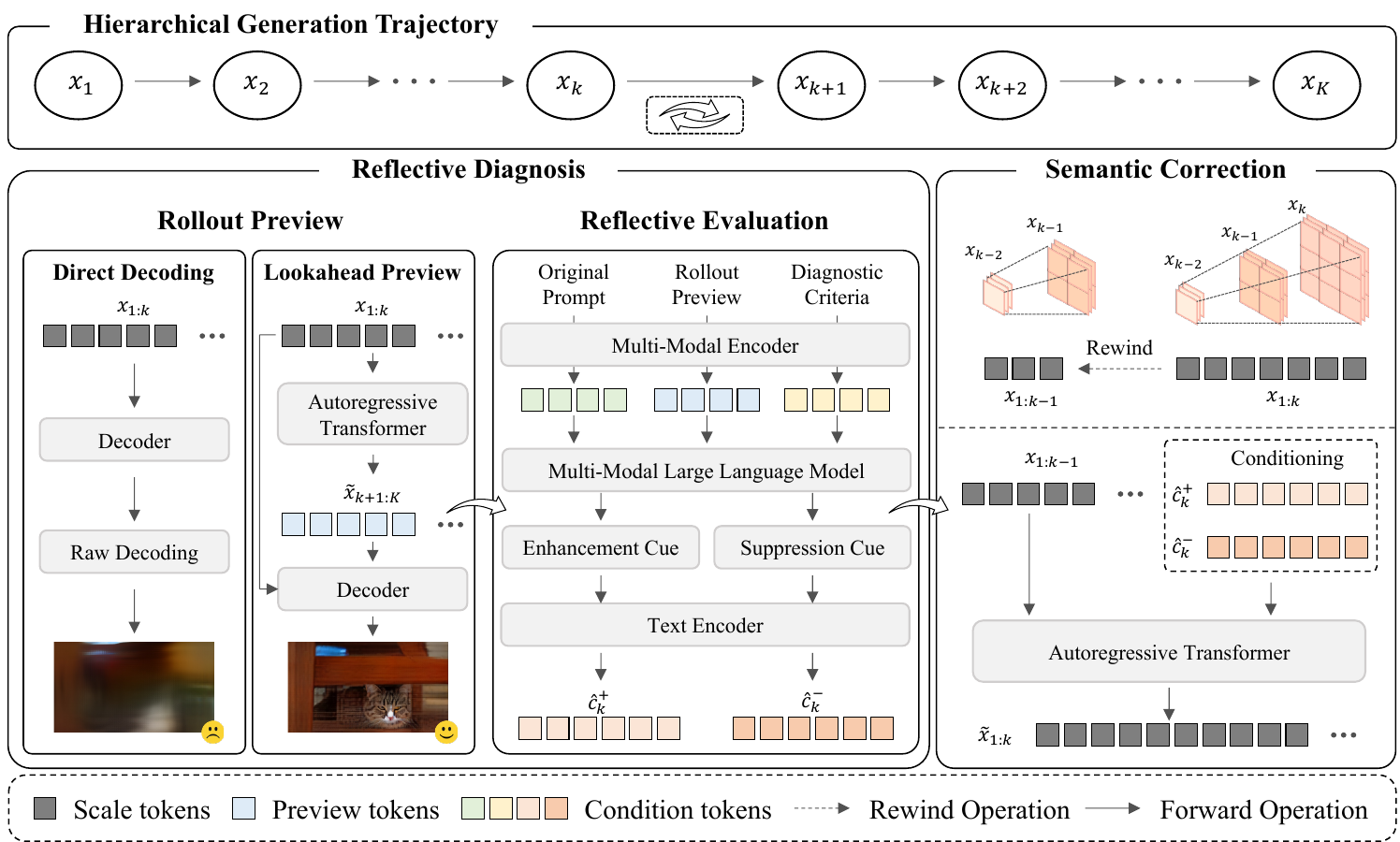}
    \caption{\textbf{Overview of \ours.}}
    \label{fig:method}
\end{figure*}


\subsection{Preliminaries}
\label{sec:prelim}
We apply \ours to \ar that predict visual tokens via next-scale prediction~\citep{tian2024var}.
Let $x_{1:K} = (x_1, \dots, x_K)$ denote the multi-scale token sequence produced by a backbone with parameters $\theta$. 
The backbone is conditioned on a text prompt $c$, which is first mapped to a continuous condition vector $\hat{c} = \tau_\phi(c)$ by a frozen text encoder $\tau_\phi$.
The backbone follows the next-scale autoregressive factorization, under which the joint distribution decomposes as
\begin{equation}
p_\theta(x_{1:K} \mid \hat{c}) \;=\; \prod_{k=1}^{K} p_\theta(x_k \mid x_{<k},\, \hat{c}),
\label{eq:ar}
\end{equation}
where each next-scale conditional $p_\theta(x_k \mid x_{<k}, \hat{c})$ predicts the tokens at scale $k$ from the preceding scales and the prompt.
Let $\mathcal{D}$ denote the decoder that maps tokens to a visual output (an image or a video), and $\mathcal{M}$ a pretrained MLLM that takes this visual output and a text query as input and returns natural-language feedback.
Sampling proceeds strictly from coarse to fine, and once $x_{<k}$ is sampled it remains fixed throughout the remaining generation, so semantic decisions made at early scales are not revisited.
Given a text prompt $c$, its encoding $\hat{c}$, and a pretrained next-scale AR backbone $p_\theta$ with decoder $\mathcal{D}$, our goal is to produce a visual output $\mathcal{D}(x_{1:K})$ whose semantics match $c$ as closely as possible, under the constraint that $\theta$, $\mathcal{D}$, and $\mathcal{M}$ remain frozen at inference.


\subsection{\stageone}
\label{sec:diagnosis}

Directly feeding the partial token map $\mathcal{D}(x_{1:k})$ to $\mathcal{M}$ is unreliable, as the decoded output at coarse scales lacks the high-frequency detail required for $\mathcal{M}$ to recognize objects, attributes, or relations.
To obtain a semantically diagnosable state, we construct a \emph{rollout preview} by extending the current trajectory to scale $K$ under the original sampling process. 
The resulting preview exposes the semantic tendency of the partially generated trajectory while preserving the macro structure already established, enabling reliable semantic diagnosis before the final output is produced.

Given $x_{1:k}$, we autoregressively sample the remaining tokens under $p_\theta$,
\begin{equation}
\tilde{x}_{k+1:K} \;\sim\; \prod_{j=k+1}^{K} p_\theta\bigl(\tilde{x}_j \mid x_{1:k}, \tilde{x}_{k+1:j-1}, \hat{c}\bigr),
\label{eq:rollout}
\end{equation}
and decode the preview $\tilde{I} = \mathcal{D}(x_{1:k}, \tilde{x}_{k+1:K})$.
$\tilde{I}$ preserves the macro structure already committed in $x_{1:k}$ while exposing the semantic consequence that the current trajectory would produce under unaltered sampling.
Because the preview follows the same autoregressive factorization as the final generation, the semantic feedback returned by $\mathcal{M}$ reflects the future semantic tendency of the current trajectory rather than an arbitrary completion, allowing for diagnostics of semantic errors before they propagate across scales.

We reflect on the current trajectory through the rollout preview $\tilde{I}$ with $\mathcal{M}$ under the original prompt $c$, asking whether each compositional element of $c$ (e.g., objects, attributes, and relations) is faithfully reflected in the emerging semantics.
The reflection produces two forms of corrective feedback: an enhancement cue $c_k^+$ that reinforces missing or weakly expressed semantics, and a suppression cue $c_k^-$ that identifies inconsistent attributes or undesired visual patterns.
Both cues are encoded by $\tau_\phi$, yielding $\hat{c}^+_{k} = \tau_\phi(c^+_{k})$ and $\hat{c}^-_{k} = \tau_\phi(c^-_{k})$, before being passed to \stagetwo for trajectory rectification.
When $c_k^+ \equiv c$ and $c_k^- = \emptyset$, no corrective signal is introduced, recovering the sampling behavior of the original backbone.


\subsection{\stagetwo}
\label{sec:correction}

Given the cues from \stageone, correcting the trajectory requires re-sampling $x_k$ itself, since the autoregressive factorization conditions every subsequent scale on $x_{1:k}$, preventing later scales from removing semantic errors already committed at $x_k$.

We therefore rewind the trajectory to $x_{1:k-1}$ and re-sample $x_k$ under the enhancement and suppression cues, steering the trajectory toward the target semantics while preserving the macro structure committed in $x_{1:k-1}$. 
Formally, the rewind operation is defined as
\begin{equation}
\mathcal{R}_{k}\bigl(x_{1:k}\bigr) \;=\; x_{1:k-1},
\label{eq:rewind}
\end{equation}
which removes the semantic commitment at $x_k$ and restores the trajectory to $x_{1:k-1}$, leaving the preceding prefix intact.

We realize semantic steering through classifier-free guidance~\citep{ho2022cfg} (CFG), treating corrective cues as additional conditions derived from diagnosis of the current trajectory rather than fixed prompts.
Let $\ell_\theta(\cdot \mid x_{<k}, \hat{c})$ denote the next-scale logits produced by the backbone under text condition $\hat{c}$, which may comprise multiple prompts attended to jointly through cross-attention.
Let$ \ell^+_{k} \equiv \ell_\theta(\cdot \mid x_{<k}, \{\hat{c}, \hat{c}^+_{k}\}), $ and $ \ell^-_{k} \equiv \ell_\theta(\cdot \mid x_{<k}, \hat{c}^-_{k}), $
denote the logits under the positive and negative conditions, respectively.
The re-sampling of $k$ takes the form
\begin{equation}
x_{k} \;\sim\; \mathrm{softmax}\bigl(\ell^-_{k} + \omega\,(\ell^+_{k} - \ell^-_{k})\bigr),
\label{eq:resample}
\end{equation}
where $\omega \ge 1$ is the CFG guidance scale.
Although Eq.~\eqref{eq:resample} follows the standard CFG formulation, the conditioning cues are generated dynamically from MLLM diagnosis of the current trajectory.

After re-sampling $x_{k}$, the backbone resumes standard scale-by-scale sampling for $j \in \{k+1, \dots, K\}$ under unmodified $p_\theta(\cdot \mid x_{<j}, \hat{c})$.
Since every subsequent scale conditions on $x_{1:k}$, rectifying $x_k$ implicitly propagates the semantic correction to all later scales.
Semantic inconsistencies may nevertheless emerge progressively and accumulate across scales along the trajectory, so a single correction step is insufficient to resolve all semantic drift.
\ours therefore applies diagnosis and correction at multiple scales along the trajectory, defined by the intervention scale set:
\begin{equation}
    \mathcal{S} =
    \bigl\{k \big|
        k = \kappa_s + i\cdot\Delta,\
        i = 0, 1, \dots,
        \bigl\lfloor
            \tfrac{\kappa_e - \kappa_s}{\Delta}
        \bigr\rfloor
    \bigr\},
    \label{eq:scale_set}
\end{equation}
where $\kappa_s = \lfloor r_s \cdot K \rfloor$ and $\kappa_e = \lfloor r_e \cdot K \rfloor$ for normalized ratios $r_s, r_e \in [0, 1)$ with $r_s < r_e$, and $\Delta \ge 1$, jointly determining where diagnosis begins, where it ends, and how frequently it is triggered along the trajectory for semantic correction.
 
\subsection{Theoretical Foundations}
\label{sec:theory}

We provide theoretical guarantees that jointly justify the design of \ours.

\begin{assumption}[Local Semantic Refinement]
\label{ass:refinement}
For all scales $k' > k$, the conditional distribution
$p_\theta(x_{k'} \mid x_{<k'}, \hat{c})$
performs only local refinement and introduces no semantic factors outside
the coarse semantic support of the committed prefix.
Formally,
\begin{equation}
    \phi(x_{1:k'})
    \subseteq
    \phi(x_{1:k})
    \qquad
    \forall\, k' > k.
\end{equation}
\end{assumption}

\begin{theorem}[Semantic Commitment]
\label{thm:commitment}
Under Assumption~\ref{ass:refinement},
if a semantic factor $s$ is absent from the committed prefix
$\phi(x_{1:k})$, then $s$ cannot appear in any continuation
$x_{k+1:K} \sim \prod_{k'=k+1}^{K}
p_\theta(x_{k'} \mid x_{<k'}, \hat{c})$.
Formally,
\begin{equation}
    s \notin \phi(x_{1:k})
    \;\Rightarrow\;
    s \notin \phi(x_{1:K}).
\end{equation}
\end{theorem}

\begin{theorem}[Rewind Expands Reachable Semantic Support]
\label{thm:support}
Under Assumption~\ref{ass:refinement}, define the reachable semantic support of a prefix $x_{1:k}$ as
\begin{equation}
    \mathcal{S}(x_{1:k})
    =
    \Bigl\{
    s \;\Bigm|\;
    \exists\, x_{k+1:K},
    \;
    s \in \phi(x_{1:K})
    \Bigr\}.
    \label{eq:support}
\end{equation}
\begin{enumerate}[label=(\roman*)]
  \item If $s \notin \mathcal{S}(x_{1:k})$, then no intervention confined to scales $k > k$ satisfies $s \in \phi(x_{1:k}, x'_{k+1:K})$.
  \item If the rewind operator $\mathcal{R}^{(1)}_{k}$ is applied and a replacement sample $x'_{k}$ satisfies $s \in \phi(x_{1:k-1}, x'_{k})$, then $s \in \mathcal{S}(x_{1:k-1}, x'_{k})$.
\end{enumerate}
\end{theorem}

\begin{remark}[Implication for Semantic Correction]
Theorems~\ref{thm:commitment} and~\ref{thm:support} together imply that correcting semantic failures requires rewinding to an earlier scale where the desired semantic factors remain reachable, motivating the rewind-and-resample mechanism of \stagetwo.
The proof is provided in Appendix \ref{app:proofs}.
\end{remark}

%% file: 4_experiments.tex
\input{tables/t2i_main}

\input{tables/t2v_main}

\subsection{Experimental Setup}
\label{sec:exp_setup}

\paragraph{Tasks and benchmarks.}
We evaluate \ours on both text-to-image and text-to-video generation to examine whether the proposed training-free intervention generalizes across \ar visual generation settings.
For text-to-image generation, we use T2I-CompBench~\citep{huang2025t2icompbenchpp}, which covers color, shape, texture, 2D-spatial relationships, 3D-spatial relationships, non-spatial relations, numeracy, and complex compositional prompts.
For text-to-video generation, we use T2V-CompBench~\citep{sun2025t2vcompbench}, which evaluates consistent attributes, dynamic attributes, spatial relations, motion, action binding, object interactions, and video-level numeracy.
These benchmarks allow us to measure whether in-generation diagnosis and correction improve fine-grained text-visual semantic alignment rather than only visual fidelity.

\paragraph{Models and baselines.}
We apply \ours to three \ars: InfinityStar~\citep{liu2025infinitystar}, Helios~\citep{yuan2026heliosrealrealtimelong}, and STAR~\citep{ma2025star}, where InfinityStar supports both tasks.
For text-to-image generation, we evaluate InfinityStar and STAR, using each unmodified model as the paired baseline.
For text-to-video generation, we evaluate InfinityStar and Helios against their respective unmodified baselines.
For each model, the baseline and \ours share the same tokenizer, decoder, model parameters, and base sampling configuration.
The setup isolates and evaluates the effect of introducing in-generation diagnosis and correction into the inference-time trajectory.

\paragraph{Implementation details.}
All experiments are conducted on NVIDIA A100 GPUs.
By default, we set the diagnosis interval to $\kappa_s=0$ and $\kappa_e=1$, and set the scale step to $\Delta=4$, which triggers diagnosis and correction along the full coarse-to-fine trajectory.
We use Qwen3-VL~\citep{bai2025qwen3} as the MLLM to evaluate rollout previews and generate enhancement and suppression cues.
To ensure a fair comparison, all models are evaluated in a training-free setting, and standard sampling and \ours use the same random seed and the same sampling configuration unless otherwise stated.

\paragraph{Evaluation metrics.}
T2I-CompBench and T2V-CompBench serve as the compositional benchmarks.
For text-to-image generation, we additionally report CLIP Score~\citep{radford2021clip}, PickScore~\citep{kirstain2023pickscore}, HPSv2~\citep{wu2023hpsv2}, and ImageReward~\citep{xu2023imagereward} to assess text-image similarity, human preference alignment.
All metrics are used to compare standard sampling and \ours under the same model.
All evaluators are used solely for evaluation purposes, and none serve as optimization objectives or reranking signals in either the baseline or \ours.

\input{tables/ablation}

\subsection{Main Results}
\label{sec:main_results}

\paragraph{Quantitative results.}
Tables~\ref{tab:t2i_main} and~\ref{tab:t2v_main} report the main results of \ours on T2I-CompBench and T2V-CompBench.
In general, \ours improves the compositional alignment across both benchmarks and across the evaluated models.
In T2I-CompBench, InfinityStar shows notable gains in shape, texture, and numeracy, while \ours in STAR improves in color, shape, texture, 3D-spatial relationships, numeracy, and complex prompts.
In T2V-CompBench, both InfinityStar and Helios show consistent gains in numeracy, spatial relations, and object interaction.
These results are consistent with the hypothesis that rollout previews enable early detection of semantic errors, and that rewind-and-resample under corrective cues can redirect the trajectory toward the target semantics before errors propagate across scales.

\ours also improves relational compositionality.
The gains are particularly visible in video generation, where spatial relations, action binding, and object interactions require consistency beyond a single frame.
For InfinityStar, \ours improves spatial relations and object interactions.
For Helios, the strongest gains appear in action binding and object interactions.
This pattern is consistent with the design of \ours, which targets and corrects semantic errors at coarse scale before they propagate through the trajectory.

The improvement is not strictly monotonic on every metric.
STAR shows small decreases on 2D-spatial and non-spatial attributes, and InfinityStar shows a small decrease on dynamic attributes in T2V-CompBench.
These mixed cases indicate that the current intervention is most reliable when the rollout preview provides sufficient visual signal for diagnosing objects, attributes, relations, and numeracy.
Fine-grained dynamic attributes and some model-specific dimensions remain challenging under the current intervention design.
Nevertheless, the overall trend supports the effectiveness of training-free \stageone and \stagetwo for improving compositional alignment without updating the model.

\paragraph{Qualitative results.}
Figure~\ref{fig:qualitative} provides qualitative comparisons on both text-to-image and text-to-video prompts.
In image generation, standard sampling often produces visually plausible images but may omit objects, confuse attributes, or bind an attribute to the wrong entity.
\ours exposes such errors through rollout previews and uses the resulting cues to strengthen missing elements and suppress conflicting ones during re-sampling.
In video generation, the baseline errors are not limited to individual frames and can propagate into incorrect actions or incoherent object interactions.
The examples show that \ours yields more stable subjects, more accurate action relationships, and more coherent interactions, which agrees with the improvements observed on T2V-CompBench.
These cases indicate that the main benefit of \ours lies in correcting semantic misalignments with respect to the target prompt, rather than overall visual quality enhancement.
Additional qualitative examples are provided in Appendix~\ref{app:qualitative}.

\subsection{Ablation and Hyperparameter Analysis}
\label{sec:ablation_schedule}

\paragraph{Ablation study.}
Table~\ref{tab:ablation} analyzes the key components of the two-stage design on InfinityStar with T2I-CompBench.
The evaluated components include \stagetwo for cue-guided re-sampling after trajectory rewind, MLLM-guided diagnosis in \stageone for identifying semantic mismatches from intermediate generation states, and rollout preview construction in \stageone for exposing a semantically readable visual proxy to the MLLM.

The ablation shows that \stagetwo alone is insufficient for stable compositional improvement.
It improves several appearance-related dimensions, but its effect on spatial relationships and numeracy is weaker, suggesting that correction requires a reliable diagnostic signal.
Notably, adding diagnosis without preview can degrade preformance on certain dimensions, suggesting that diagnosis on incomplete intermerdiate states introduces unreliable cues.
The full system achieves the most stable overall improvement, especially on color, shape, texture, 2D-spatial relationships, non-spatial attributes, and numeracy.
This result supports the cooperative design of \ours: rollout preview provides readable evidence, MLLM-guided diagnosis converts that evidence into corrective cues, and correction steers the trajectory toward the target semantics under corrective cues.

\begin{figure}[t]
\centering
\includegraphics[width=\columnwidth]{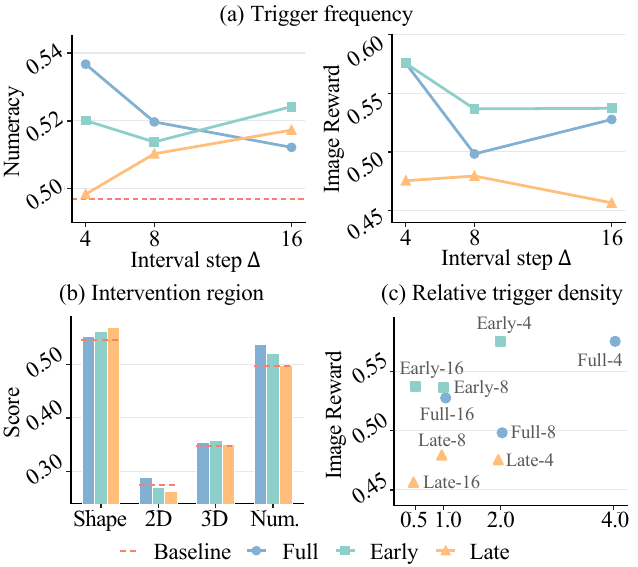}
\caption{
\textbf{Schedule behavior on InfinityStar with T2I-CompBench.}
Full, Early, and Late denote $[r_s,r_e]=[0,1]$, $[0,0.5]$, and $[0.5,1]$, respectively.
(a) shows the effect of trigger frequency $\Delta$.
(b) compares high-frequency intervention regions.
(c) relates relative trigger density to ImageReward.
}
\label{fig:schedule_main}
\end{figure}

\paragraph{Intervention schedule.}
We study how the intervention range $(\kappa_s, \kappa_e)$ and trigger interval $\Delta$ affect compositional generation on InfinityStar with T2I-CompBench, with representative trends summarized in Figure~\ref{fig:schedule_main}.
Smaller $\Delta$ consistently improves numeracy and ImageReward across all intervention regions (Figure~\ref{fig:schedule_main}a), and full-range intervention achieves more balanced gains across compositional dimensions than early or late intervention alone (Figure~\ref{fig:schedule_main}b).
Higher trigger density is further associated with better ImageReward, with early intervention dominating late intervention at comparable density (Figure~\ref{fig:schedule_main}c).
We therefore adopt the full-range $\Delta=4$ schedule as the default setting, which provides the strongest overall balance at a moderate inference overhead.
Full schedule results are provided in Appendix~\ref{appendix:schedule}.

\input{tables/budget}

\subsection{Efficiency Analysis}
\label{sec:efficiency}

\ours introduces additional inference-time overhead from rollout preview construction and MLLM-based diagnosis, but the overall cost remains below that of Best-of-2 sampling.
As shown in Table~\ref{tab:budget}, \ours requires $1.63\times$ the inference time of standard sampling, compared with $2.00\times$ for Best-of-2, while achieving higher compositional consistency and human preference alignment.
This indicates that the performance gain stems from targeted semantic correction rather than increased sampling diversity.
The model memory footprint remains unchanged; the additional overhead stems entirely from the MLLM loaded during diagnosis.
Full metric breakdowns and memory statistics are provided in Appendix~\ref{appendix:efficiency}.

%% file: tables/t2i_main.tex
\begin{table*}[t]
\centering
\small
\setlength{\tabcolsep}{6pt}
\begin{tabular}{llcccccccc}
\toprule
Model & Method & Color & Shape & Texture & 2D Spatial & 3D Spatial & Non-spatial & Numeracy & Complex \\
\midrule
\multirow{2}{*}{InfinityStar}
& Standard & 0.8121 & 0.5459 & 0.7384 & 0.2739 & 0.3466 & 0.2995 & 0.4970 & 0.3819 \\
& \ours & \textbf{0.8288} & \textbf{0.6348} & \textbf{0.7676} & \textbf{0.2760} & \textbf{0.3535} & \textbf{0.2998} & \textbf{0.5436} & \textbf{0.3862} \\
\midrule
\multirow{2}{*}{STAR}
& Standard & 0.5327 & 0.4270 & 0.5293 & \textbf{0.1952} & 0.3676 & \textbf{0.3110} & 0.4589 & 0.3418 \\
& \ours & \textbf{0.5543} & \textbf{0.4386} & \textbf{0.5299} & 0.1942 & \textbf{0.3808} & 0.3109 & \textbf{0.4680} & \textbf{0.3437} \\
\bottomrule
\end{tabular}%
\caption{\textbf{Main results on T2I-CompBench.} For each model, \ours is compared with the corresponding standard sampling baseline under the same model sampling configuration.}
\label{tab:t2i_main}
\end{table*}

%% file: tables/t2v_main.tex
\begin{table*}[t]
\centering
\small
\setlength{\tabcolsep}{4pt}
\begin{tabular}{llccccccc}
\toprule
Model & Method & Consistent Attr. & Dynamic Attr. & Spatial & Motion & Action Binding & Object Interaction & Numeracy \\
\midrule
\multirow{2}{*}{InfinityStar}
& Standard & 0.8657 & \textbf{0.0274} & 0.6248 & 0.3212 & 0.7258 & 0.6770 & 0.3408 \\
& \ours & \textbf{0.8826} & 0.0265 & \textbf{0.6614} & \textbf{0.3383} & \textbf{0.7285} & \textbf{0.7169} & \textbf{0.3745} \\
\midrule
\multirow{2}{*}{Helios}
& Standard & 0.7550 & 0.0232 & 0.5484 & 0.2683 & 0.6496 & 0.5546 & 0.3203 \\
& \ours & \textbf{0.7875} & \textbf{0.0296} & \textbf{0.5857} & \textbf{0.2793} & \textbf{0.7091} & \textbf{0.6741} & \textbf{0.3508} \\
\bottomrule
\end{tabular}%
\caption{\textbf{Main results on T2V-CompBench.} \ours improves most video compositional dimensions, especially action binding, object interactions, and video-level numeracy.}
\label{tab:t2v_main}
\end{table*}

%% file: tables/ablation.tex
\begin{table*}[t]
\centering
\small
\setlength{\tabcolsep}{4pt}
\begin{tabular}{ccc|cccccccc}
\toprule
Correction & Diagnosis & Preview & Color & Shape & Texture & 2D Spatial & 3D Spatial & Non-spatial & Numeracy & Complex \\
\midrule
 &  &  & 0.8121 & 0.5459 & 0.7384 & 0.2739 & 0.3466 & 0.2995 & 0.4970 & 0.3819 \\
\checkmark &  &  & 0.8137 & 0.5633 & 0.7517 & 0.2604 & 0.3482 & 0.2984 & 0.4934 & 0.3874 \\
\checkmark & \checkmark &  & 0.8068 & 0.5439 & 0.7420 & 0.2230 & \textbf{0.3621} & 0.2995 & 0.5135 & \textbf{0.3919} \\
\checkmark & \checkmark & \checkmark & \textbf{0.8288} & \textbf{0.6348} & \textbf{0.7676} & \textbf{0.2760} & 0.3535 & \textbf{0.2998} & \textbf{0.5436} & 0.3862 \\
\bottomrule
\end{tabular}%
\caption{\textbf{Ablation study on InfinityStar with T2I-CompBench.} Preview denotes rollout preview construction and Diagnosis denotes MLLM-guided semantic diagnosis(\stageone). Correction denotes cue-guided re-sampling after trajectory rewind(\stagetwo). Each row progressively enables additional stages of the proposed framework.}
\label{tab:ablation}
\end{table*}

%% file: tables/budget.tex

\begin{table}[t]
\centering
\small
\setlength{\tabcolsep}{1.5pt}
\begin{tabular}{lcccc}
\toprule
Method & Time $\downarrow$ & Mem. $\downarrow$ & Consist. Attr. $\uparrow$ & IR $\uparrow$ \\
\midrule
Baseline & $1.00\times$ & 59.40 & 13.07 & 0.2811 \\
Baseline Best-of-2 & $2.00\times$ & 59.40 & 13.23 & 0.2802 \\
\ours & $\mathbf{1.63\times}$ & 77.44 & \textbf{14.20} & \textbf{1.1452} \\
\bottomrule
\end{tabular}%
\caption{
\textbf{Budget-aware comparison. }
\ours achieves higher compositional consistency and human preference alignment than Best-of-2 sampling, while requiring lower inference cost.
Memory is reported in GB.
Full metric breakdowns are provided in Appendix~\ref{appendix:efficiency}.
}
\vspace{-0.4cm}
\label{tab:budget}
\end{table}

%% file: 6_conclusion.tex
We presented \ours, a training-free framework for in-generation semantic correction in next-scale autoregressive visual generation, demonstrating that MLLM feedback can be effectively incorporated into the AR sampling process to diagnose and rectify semantic errors before they propagate across scales.
Experiments on T2I-CompBench and T2V-CompBench show consistent improvements in compositional alignment across image and video backbones without any additional training.
We hope this work offers a modest step toward in-generation correction as an underexplored direction for autoregressive visual generation, and may serve as a reference for future work on real-time semantic calibration during AR sampling, including more efficient feedback mechanisms, adaptive intervention schedules, and extensions to broader generative architectures.

%% file: 7_appendix.tex
\section{Proofs of Theorems}
\label{app:proofs}
\subsection{Proof of Theorem~\ref{thm:commitment}}
\label{sec:proof-commitment}

\begin{theorem}[Semantic Commitment, Restatement]
Under Assumption~\ref{ass:refinement}, if a semantic factor $s$ is absent from the committed prefix $\phi(x_{1:k})$, then $s$ cannot appear in any continuation $x_{k+1:K} \sim \prod_{k'=k+1}^{K} p_\theta(x_{k'} \mid x_{<k'}, \hat{c})$.
Formally,
\begin{equation}
    s \notin \phi(x_{1:k})
    \;\Rightarrow\;
    s \notin \phi(x_{1:K}).
\end{equation}
\end{theorem}

We use the next-scale AR factorization in Eq.~\eqref{eq:ar} from Section~\ref{sec:prelim}, under which sampling proceeds strictly from coarse to fine and $x_{<k}$ remains fixed once sampled.
We further invoke the local-refinement assumption stated in the theorem: for all $k' > k$, the conditional $p_\theta(x_{k'} \mid x_{<k'}, \hat{c})$ introduces no semantic factors outside the support of $\phi(x_{1:k})$, i.e., $\phi(x_{1:k'}) \subseteq \phi(x_{1:k})$ for all $k' > k$.

We proceed by induction on $k'$ over $\{k+1, \dots, K\}$.

\textbf{Base case ($k' = k+1$).}
By the local-refinement assumption applied at $k' = k+1$,
\begin{equation}
    \phi(x_{1:k+1}) \;\subseteq\; \phi(x_{1:k}).
\end{equation}
Since $s \notin \phi(x_{1:k})$ by hypothesis, it follows that $s \notin \phi(x_{1:k+1})$.

\textbf{Inductive step.}
Suppose $s \notin \phi(x_{1:k'})$ for some $k < k' < K$.
Because $x_{<k'}$ remains fixed under the AR factorization in Eq.~\eqref{eq:ar} once committed, applying the local-refinement assumption at scale $k'+1$ gives
\begin{equation}
    \phi(x_{1:k'+1}) \;\subseteq\; \phi(x_{1:k'}),
\end{equation}
where the containment holds because $p_\theta(x_{k'+1} \mid x_{<k'+1}, \hat{c})$ introduces no semantic factors outside those already present in $\phi(x_{1:k'}) \supseteq \phi(x_{1:k})$.
Since $s \notin \phi(x_{1:k'})$, it follows immediately that
$s \notin \phi(x_{1:k'+1})$.

Applying the induction through all scales $k < k' \leq K$ yields
\begin{equation}
    \phi(x_{1:K}) \;\subseteq\; \phi(x_{1:k}),
\end{equation}
and therefore $s \notin \phi(x_{1:K})$ for any continuation
$x_{k+1:K} \sim \prod_{k'=k+1}^{K} p_\theta(x_{k'} \mid x_{<k'}, \hat{c})$.
$\square$

The semantic-commitment result shows that semantic factors absent at commitment scale $k$ are irrecoverable by any continuation under the local-refinement assumption, which motivates the rewind-and-resample mechanism in \stagetwo (Section~\ref{sec:correction}): a forward-only continuation cannot recover $s$, so \ours rewind to an earlier scale and re-sample $x_{k}$ under cue-guided conditioning.

\subsection{Proof of Theorem~\ref{thm:support}}
\label{sec:proof-support}

\begin{theorem}[Rewind Expands Reachable Semantic Support, Restatement]
Under Assumption~\ref{ass:refinement}, define the reachable semantic support of a prefix $x_{1:k}$ as
\begin{equation}
    \mathcal{S}(x_{1:k})
    =
    \Bigl\{
    s \;\Bigm|\;
    \exists\, x_{k+1:K},
    \;
    s \in \phi(x_{1:K})
    \Bigr\}.
\end{equation}
\begin{enumerate}[label=(\roman*)]
  \item If $s \notin \mathcal{S}(x_{1:k})$, then no intervention confined to scales $k' > k$ satisfies $s \in \phi(x_{1:k}, x'_{k+1:K})$.
  \item If the rewind operator $\mathcal{R}_{k}$ is applied and a replacement sample $x'_{k}$ satisfies
  $s \in \phi(x_{1:k-1}, x'_{k})$, then $s \in \mathcal{S}(x_{1:k-1}, x'_{k})$.
\end{enumerate}
\end{theorem}

We use the next-scale AR factorization in Eq.~\eqref{eq:ar} from Section~\ref{sec:prelim}, the rewind operator $\mathcal{R}_{k}$ defined in Eq.~\eqref{eq:rewind}, and the reachable semantic support $\mathcal{S}(\cdot)$ defined in Eq.~\eqref{eq:support}.
We further invoke the local-refinement assumption from Theorem~\ref{thm:commitment}: for all $k' > k$, each conditional $p_\theta(x_{k'} \mid x_{<k'}, \hat{c})$ introduces no semantic factors outside $\phi(x_{1:k})$.
By Theorem~\ref{thm:commitment}, the local-refinement assumption implies $\phi(x_{1:K}) \subseteq \phi(x_{1:k})$ for any continuation $x_{k+1:K} \sim \prod_{k'=k+1}^{K} p_\theta(x_{k'} \mid x_{<k'}, \hat{c})$.
We prove parts~(i) and~(ii) separately.

\medskip
\noindent\textbf{Proof of part~(i).}

Suppose for contradiction that there exists an intervention $x'_{k+1:K}$ confined to scales $k' > k$ such that $s \in \phi(x_{1:k}, x'_{k+1:K})$.
By definition of $\mathcal{S}$ in Eq.~\eqref{eq:support}, the intervention would require
\begin{equation}
    s \;\in\; \phi(x_{1:k}, x'_{k+1:K})
    \;\subseteq\;
    \bigcup_{x'_{k+1:K}} \phi(x_{1:K})
    \;=\;
    \mathcal{S}(x_{1:k}).
\end{equation}
The conclusion contradicts the hypothesis $s \notin \mathcal{S}(x_{1:k})$.
Therefore no intervention confined to scales $k' > k$ can introduce $s$ into the semantic content of the final output, regardless of how the continuation $x'_{k+1:K}$ is chosen or modulated.
\hfill$\square_{(\mathrm{i})}$

\medskip
\noindent\textbf{Proof of part~(ii).}

Apply $\mathcal{R}_{k}$ to the committed trajectory.
By Eq.~\eqref{eq:rewind}, the rewind operator discards $x_{k}$ while leaving $x_{1:k-1}$ unchanged:
\begin{equation}
    \mathcal{R}_{k}(x_{1:k}) \;=\; x_{1:k-1}.
\end{equation}
Let $x'_{k}$ be a replacement sample satisfying $s \in \phi(x_{1:k-1}, x'_{k})$ by hypothesis.
We must show $s \in \mathcal{S}(x_{1:k-1}, x'_{k})$, i.e., that there exists some continuation $x'_{k+1:K}$ such that $s \in \phi(x_{1:k-1}, x'_{k}, x'_{k+1:K})$.

Take the trivial continuation $x'_{k+1:K} = \emptyset$ (i.e., $K = k$), or, for $K > k$, any continuation that performs only local refinement.
Under the local-refinement assumption applied to the new prefix
$(x_{1:k-1}, x'_{k})$,
\begin{equation}
    \phi\bigl(x_{1:k-1}, x'_{k}, x'_{k+1:K}\bigr)
    \;\subseteq\;
    \phi\bigl(x_{1:k-1}, x'_{k}\bigr),
\end{equation}
where the containment follows from Theorem~\ref{thm:commitment} applied to the prefix $(x_{1:k-1}, x'_{k})$ in place of $x_{1:k}$.
Since $s \in \phi(x_{1:k-1}, x'_{k})$ by hypothesis, it follows that $s$ is not excluded by any subsequent local-refinement step.
The factor $s$ therefore persists in $\phi(x_{1:k-1}, x'_{k}, x'_{k+1:K})$ for any such continuation, which witnesses $s \in \mathcal{S}(x_{1:k-1}, x'_{k})$.
\hfill$\square_{(\mathrm{ii})}$

\medskip
Together, parts~(i) and~(ii) establish that a single-scale rewind followed by guided re-sampling is both necessary and sufficient to recover the reachable semantic support, providing the principled basis for \stagetwo.

\section{\ours Sampling Procedure}
\label{sec:appendix_algorithm}

Algorithm~\ref{alg:gazer} presents the complete \ours sampling loop, expanding the overview in Section~\ref{methods} with full algorithmic detail.

\begin{algorithm}[t]
\caption{\ours sampling loop}
\label{alg:gazer}
\begin{algorithmic}[1]
\REQUIRE $c,\,\tau_\phi,\, p_\theta,\, \mathcal{M},\, \mathcal{D},\,
         \kappa_s,\, \kappa_e,\, \Delta,\, \omega$
\ENSURE $I_K$
\STATE $\hat{c} \leftarrow \tau_\phi(c)$
\STATE $\mathcal{S} \leftarrow \{\kappa_s + i\cdot\Delta\}_{i=0}^{\lfloor(\kappa_e-\kappa_s)/\Delta\rfloor}$
\FOR{$k = 1, \dots, K$}
  \STATE $x_k \sim p_\theta(\cdot \mid x_{<k},\, \hat{c})$
  \IF{$k \in \mathcal{S}$}
    \STATE $\tilde{x}_{k+1:K}\!\sim\!
           \prod_{j=k+1}^{K}
           p_\theta(\cdot|x_{1:k},\tilde{x}_{k+1:j-1},\hat{c})$
           \hfill$\triangleright${Eq.~\eqref{eq:rollout}}
    \STATE $\tilde{I} \leftarrow \mathcal{D}(x_{1:k},\,\tilde{x}_{k+1:K})$
    \STATE $(c^+_k,\, c^-_k) \leftarrow \mathcal{M}(\tilde{I},\, c)$
    \IF{$c^+_k \neq c$ \textbf{or} $c^-_k \neq \emptyset$}
      \STATE $\hat{c}^+_k \leftarrow \tau_\phi(c^+_k);\quad \hat{c}^-_k \leftarrow \tau_\phi(c^-_k)$
      \STATE $x_{1:k-1} \leftarrow \mathcal{R}_k(x_{1:k})$
             \hfill$\triangleright${Eq.~\eqref{eq:rewind}}
      \STATE $\ell^+_k \leftarrow
             \ell_\theta(\cdot \mid x_{<k},\,
             \{\hat{c},\,\hat{c}^+_k\})$
      \STATE $\ell^-_k \leftarrow
             \ell_\theta(\cdot \mid x_{<k},\, \hat{c}^-_k)$
      \STATE $x_k \sim
             \mathrm{softmax}\bigl(
               \ell^-_k + \omega\,(\ell^+_k - \ell^-_k)
             \bigr)$
             \hfill$\triangleright${Eq.~\eqref{eq:resample}}
    \ENDIF
  \ENDIF
\ENDFOR
\RETURN $I_K \leftarrow \mathcal{D}(x_{1:K})$
\end{algorithmic}
\end{algorithm}

\section{Limitations On Next-Token Prediction Models}
\label{appendix:failure}

\begin{figure}[t]
\centering
\includegraphics[width=\columnwidth]{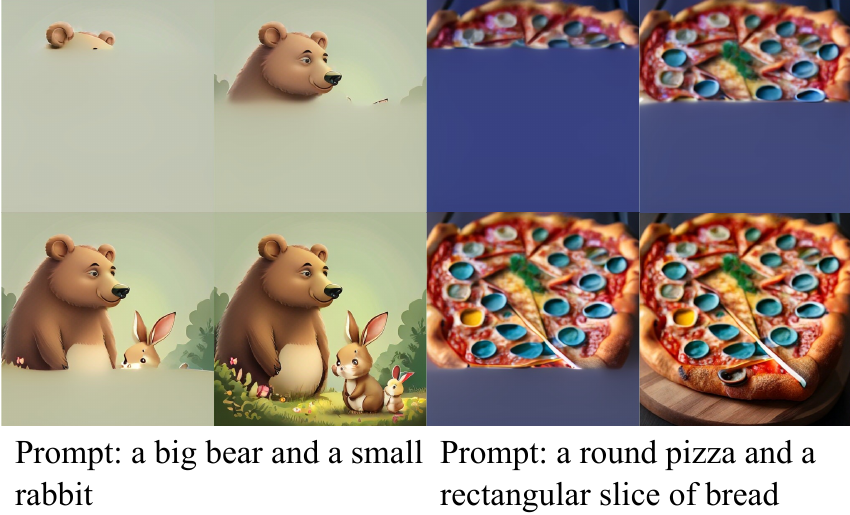}
\caption{\textbf{Failure mode.} Rollout previews at four intervention points on a next-token prediction model, where spatially truncated intermediate states prevent reliable MLLM diagnosis.}
\label{fig:failure}
\end{figure}

\ours relies on rollout previews to produce semantically readable intermediate states, but next-token prediction models generate tokens in raster-scan order, yielding spatially truncated partial images that lack global semantic structure and cannot support reliable MLLM diagnosis.

In next-scale prediction, each intermediate state corresponds to a complete but low-resolution image, so a rollout preview only needs to fill in finer details to produce a globally coherent output. In next-token prediction, the intermediate state covers only the upper portion of the image in raster-scan order, and the decoded output is spatially incomplete, leaving the MLLM unable to assess the overall semantic composition.
The rewind mechanism faces an analogous mismatch. Rewinding in next-scale prediction discards the last scale while preserving a complete low-resolution prefix. Rewinding in next-token prediction returns to an earlier token position, producing a spatially incoherent prefix from which subsequent generation cannot reliably recover the macro structure.

Figure~\ref{fig:failure} illustrates this failure mode. The four previews shown for each prompt correspond to rollout previews at successive intervention scales. For \textit{a round pizza and a rectangular slice of bread}, early previews consist largely of uninformative regions, providing no signal for detecting the missing bread. For \textit{a big bear and a small rabbit}, the absent rabbit remains invisible in the previews until late scales, by which point the semantic commitment is already difficult to reverse.
Extending \ours to next-token prediction models would require an alternative preview mechanism that produces globally coherent visual predictions from partial raster-scan sequences, which we leave for future work.

\section{Additional Experimental Analysis}
\label{sec:appendix_exp_details}

\subsection{Rollout Preview Readability Validation}
\label{appendix:preview_readability}

Table~\ref{tab:preview_readability}
evaluates the diagnostic readability of rollout previews.
The comparison is restricted to visual states available before the final output
is committed, namely direct intermediate decoding and rollout preview.
Final outputs are omitted because they cannot be queried for intervention-time
diagnosis.
For each visual state, Qwen3-VL judges whether prompt-relevant objects,
attributes, relations, and numeracy/action constraints are recognizable.

Across both T2I-CompBench and T2V-CompBench, rollout previews yield higher
recognizability than direct intermediate decoding across all evaluated semantic
dimensions.
The improvement is especially relevant for \stageone, where the MLLM must
diagnose a still-evolving trajectory rather than a completed sample.
The readability validation therefore supports the use of rollout preview as a
semantic proxy for intermediate AR states.
Instead of asking the MLLM to inspect a partially decoded and often ambiguous
state, \ours evaluates a forward rollout that exposes more prompt-relevant
visual evidence while preserving the ability to intervene before final
generation.

\begin{table}[t]
\centering
\small
\resizebox{\columnwidth}{!}{%
\begin{tabular}{llccccc}
\toprule
Benchmark & State & Object & Attribute & Relation & Num./Action & Avg. \\
\midrule
\multirow{2}{*}{T2I-CompBench}
& Raw & 0.278 & 0.308 & 0.263 & 0.251 & 0.286 \\
& Preview & \textbf{0.409} & \textbf{0.415} & \textbf{0.383} & \textbf{0.361} & \textbf{0.405} \\
\midrule
\multirow{2}{*}{T2V-CompBench}
& Raw & 0.211 & 0.213 & 0.200 & 0.181 & 0.205 \\
& Preview & \textbf{0.324} & \textbf{0.321} & \textbf{0.316} & \textbf{0.289} & \textbf{0.316} \\
\bottomrule
\end{tabular}%
}
\caption{\textbf{Rollout preview readability validation under the current \ours setting.} Raw denotes direct intermediate decoding, and Preview denotes rollout preview. Qwen3-VL evaluates whether prompt-relevant semantic constraints are recognizable from each state.}
\label{tab:preview_readability}
\end{table}

\subsection{Intervention Schedule Sweep}
\label{appendix:schedule}

Table~\ref{tab:schedule} summarizes the schedule sweep on InfinityStar with T2I-CompBench.
The schedule is controlled by the diagnosis interval boundaries $\kappa_s$ and $\kappa_e$ and the scale step $\Delta$.
Figure~\ref{fig:schedule_heatmap} further shows task-level evaluator scores across all schedules and evaluators.

The results show that the preferred intervention window depends on the evaluated compositional dimension.
(i) \textit{Shape.}
Shape scores vary only mildly across schedules.
All non-baseline schedules improve over standard sampling, and the highest score is achieved by a full-range schedule with $\Delta=16$.
The shape results indicate that shape correction is relatively robust to the exact intervention interval.
(ii) \textit{Numeracy.}
Numeracy is more sensitive to intervention density, with the full-range schedule at $\Delta=4$ giving the best score among all schedules.
The numeracy results support using frequent corrections when the target prompt requires multiple object instances to remain separable throughout generation.
(iii) \textit{3D spatial reasoning.}
The best 3D Spatial score comes from a late-interval schedule with $\Delta=16$, suggesting that depth-related corrections benefit from previews that already contain sufficient mid- and fine-scale structure.

The dimension-specific variation indicates a practical trade-off.
Early schedules provide earlier semantic steering, late schedules operate on more visually informative previews, and full-range schedules cover both regimes.
We use the full-range schedule with $\Delta=4$ as a balanced default because it gives the best 2D Spatial and Numeracy scores while improving all four reported dimensions over the baseline.
When computation is constrained, larger $\Delta$ values are reasonable for shape-oriented prompts.
When numeracy or mixed compositionality is central, a smaller $\Delta$ over the full interval is preferable.
For depth-sensitive spatial prompts, a late-biased interval can be competitive.

\begin{table}[t]
\centering
\small
\resizebox{\columnwidth}{!}{%
\begin{tabular}{lccc|cccc}
\toprule
Schedule & $\kappa_s$ & $\kappa_e$ & $\Delta$ & Shape & 2D Spatial & 3D Spatial & Numeracy \\
\midrule
Baseline & -- & -- & -- & 0.5459 & 0.2739 & 0.3466 & 0.4970 \\
Full-4 & 0 & 1 & 4 & 0.5525 & \textbf{0.2874} & 0.3546 & \textbf{0.5367} \\
Late-8 & 0.5 & 1 & 8 & 0.5582 & 0.2852 & 0.3594 & 0.5103 \\
Late-16 & 0.5 & 1 & 16 & 0.5565 & 0.2787 & \textbf{0.3603} & 0.5173 \\
Early-4 & 0 & 0.5 & 4 & 0.5613 & 0.2698 & 0.3575 & 0.5201 \\
Early-8 & 0 & 0.5 & 8 & 0.5598 & 0.2791 & 0.3496 & 0.5138 \\
Full-8 & 0 & 1 & 8 & 0.5697 & 0.2578 & 0.3543 & 0.5197 \\
Early-16 & 0 & 0.5 & 16 & 0.5656 & 0.2648 & 0.3405 & 0.5242 \\
Full-16 & 0 & 1 & 16 & \textbf{0.5708} & 0.2626 & 0.3438 & 0.5122 \\
Late-4 & 0.5 & 1 & 4 & 0.5687 & 0.2618 & 0.3496 & 0.4984 \\
\bottomrule
\end{tabular}%
}
\caption{\textbf{Schedule sweep on InfinityStar with T2I-CompBench.} Full, Early, and Late denote diagnosis intervals $[\kappa_s,\kappa_e]=[0,1]$, $[0,0.5]$, and $[0.5,1]$.}
\label{tab:schedule}
\end{table}

\begin{figure*}[t]
\centering
\includegraphics[width=\textwidth]{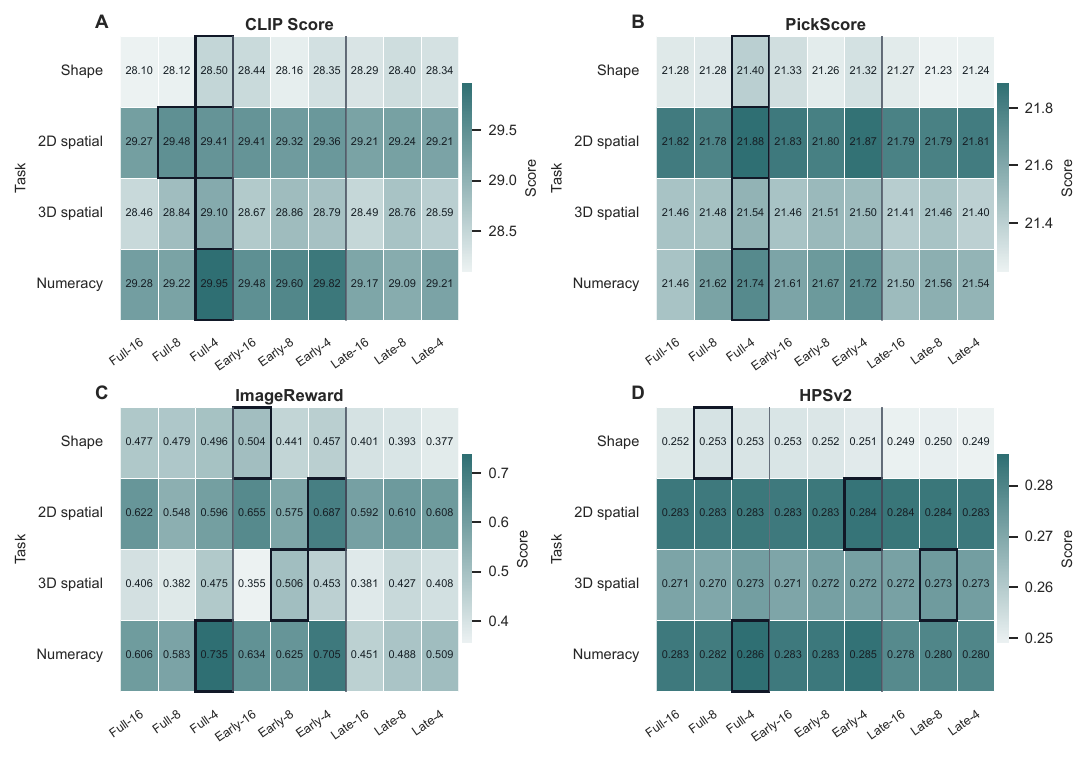}
\caption{\textbf{Task-level evaluator scores across T2I intervention schedules.} Each heatmap corresponds to one evaluator, with rows denoting T2I-CompBench task groups and columns denoting intervention schedules. Cell colors and annotations report the raw evaluator scores. Black boxes mark the best schedule for each task and evaluator pair.}
\label{fig:schedule_heatmap}
\end{figure*}

\subsection{Inference Cost and Detailed Efficiency Analysis}
\label{appendix:efficiency}

\input{tables/efficiency}

Table~\ref{tab:efficiency} provides a detailed breakdown of inference cost, memory usage, and generation quality.
Table~\ref{tab:efficiency_breakdown} further decomposes the full \ours pipeline into its main runtime components.

Compared with standard sampling, \ours increases inference time by approximately $63\%$ due to the additional rollout-preview and diagnosis stages.
Nevertheless, the total runtime remains substantially lower than Best-of-2 sampling under a comparable compute budget.
The module-level breakdown in Table~\ref{tab:efficiency_breakdown} clarifies where the additional cost is spent.
Among the intervention-specific components, MLLM evaluation is the dominant term, taking 44.91 seconds on average and accounting for $27.46\%$ of the full pipeline.
Preview construction is smaller at 12.86 seconds, while reflection adds only 1.87 seconds.
The cost is therefore concentrated in visual evaluation by the MLLM rather than in the subsequent reflection step.
Most wall-clock time still comes from remaining generation and decoding, which takes 103.93 seconds and accounts for $63.54\%$ of the pipeline.
The breakdown suggests that future efficiency gains should primarily come from reducing the number or cost of MLLM evaluations, for example through adaptive triggering, preview reuse, or batched evaluation across intervention points.

Importantly, the generator-side peak memory remains almost unchanged
(59.64 GB vs.\ 59.40 GB), indicating that the proposed intervention mechanism does not introduce significant overhead to the frozen model itself.
The additional memory consumption mainly originates from the MLLM used for trajectory diagnosis.

Despite operating under a lower inference budget than Best-of-2 sampling, \ours achieves consistently stronger performance across all compositional and preference-related metrics, suggesting that trajectory correction through semantic feedback is substantially more effective than increasing sampling diversity through repeated generation.

\subsection{Qualitative Effect of Rollout Preview}
\label{appendix:preview}

\begin{figure*}[t]
\centering
\includegraphics[width=\textwidth]{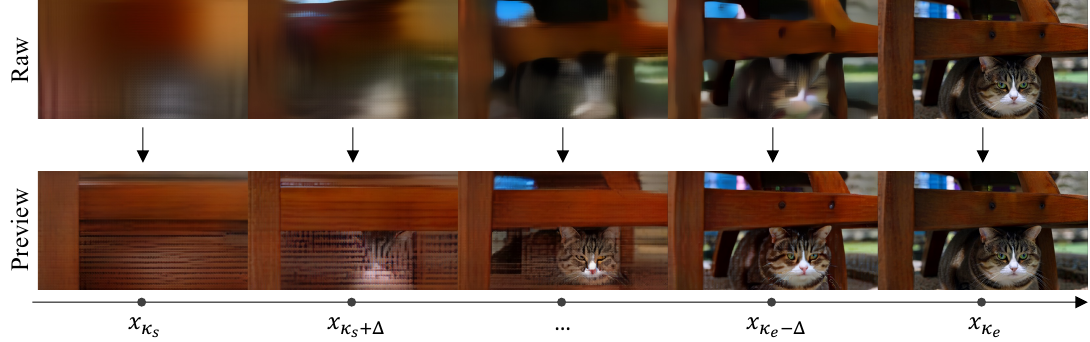}
\caption{\textbf{Qualitative Effect of Rollout Preview.} At each intervention scale $k$ along the coarse-to-fine trajectory, the raw decoded state $\mathcal{D}(x_{1:k})$ (top) is compared with the rollout preview $\mathcal{D}(x_{1:k}, \tilde{x}_{k+1:K})$ (bottom). The preview provides a semantically coherent visual prediction that enables reliable MLLM diagnosis, whereas the raw intermediate state lacks sufficient visual semantics for accurate evaluation. Prompt: 'A cat sitting under a wooden chair.'}
\label{fig:case_preview}
\end{figure*}

Figure~\ref{fig:case_preview} compares raw intermediate decoding with rollout previews at selected scales along the coarse-to-fine trajectory. At each intervention scale $k$, the raw decoded state $\mathcal{D}(x_{1:k})$ provides only a low-resolution, semantically incomplete representation, from which the MLLM cannot reliably assess object presence, attributes, or compositional relations. The rollout preview $\mathcal{D}(x_{1:k}, \tilde{x}_{k+1:K})$, by contrast, extends the partial trajectory to full resolution under the original sampling distribution, yielding a globally coherent visual prediction that faithfully reflects the semantic tendency of the current generation. This enables the MLLM to produce accurate diagnostic feedback, which \stageone converts into enhancement and suppression cues for subsequent trajectory correction.

\subsection{More Qualitative Results}
\label{app:qualitative}

\begin{figure*}[t]
\centering
\includegraphics[width=\textwidth]{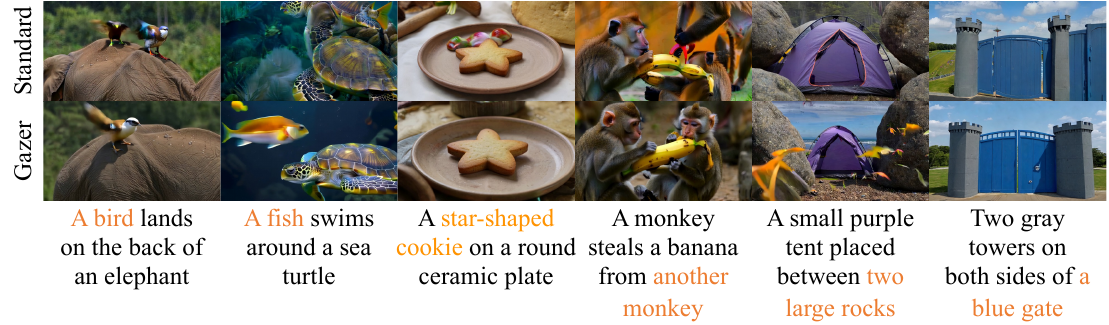}
\caption{
\textbf{Qualitative comparisons with InfinityStar on compositional text-to-image generation.}
}
\label{fig:qual_infstar_t2i}
\end{figure*}

\begin{figure*}[t]
\centering
\includegraphics[width=\textwidth]{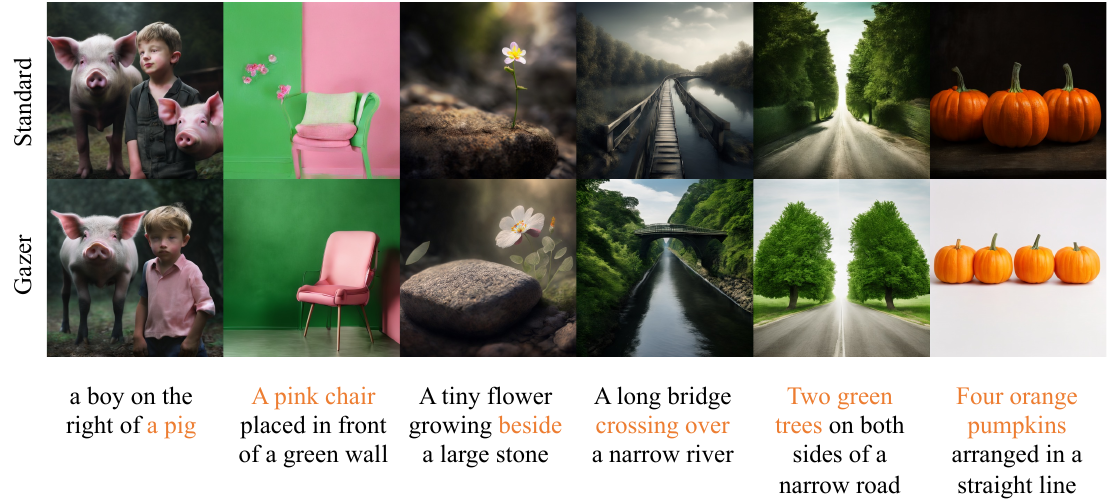}
\caption{
\textbf{Qualitative comparisons with STAR on compositional text-to-image generation.}
}
\label{fig:qual_star_t2i}
\end{figure*}

\begin{figure*}[t]
\centering
\includegraphics[width=\textwidth]{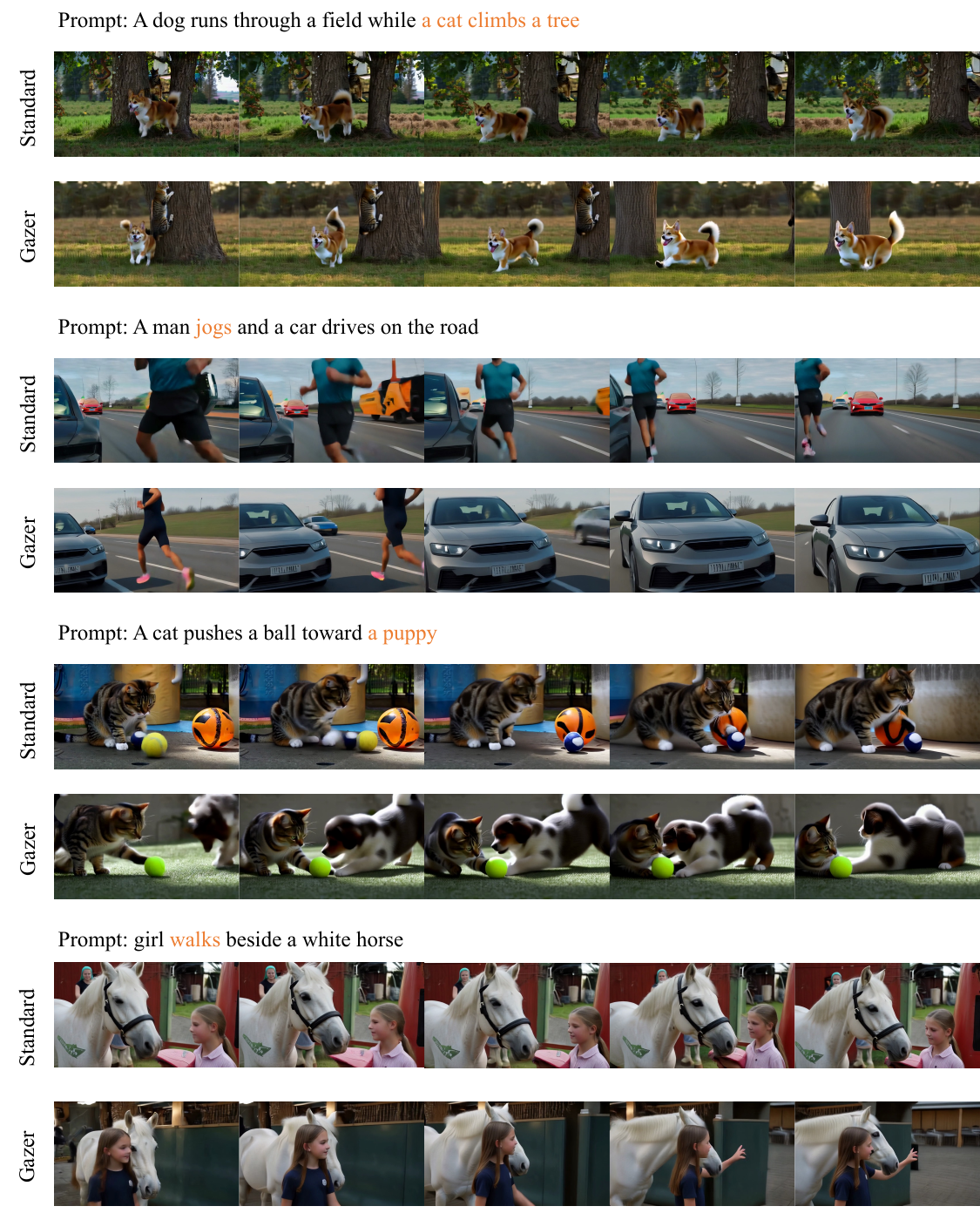}
\caption{
\textbf{Qualitative comparisons with InfinityStar on compositional text-to-video generation.}
}
\label{fig:qual_infstar_t2v}
\end{figure*}

\begin{figure*}[t]
\centering
\includegraphics[width=\textwidth]{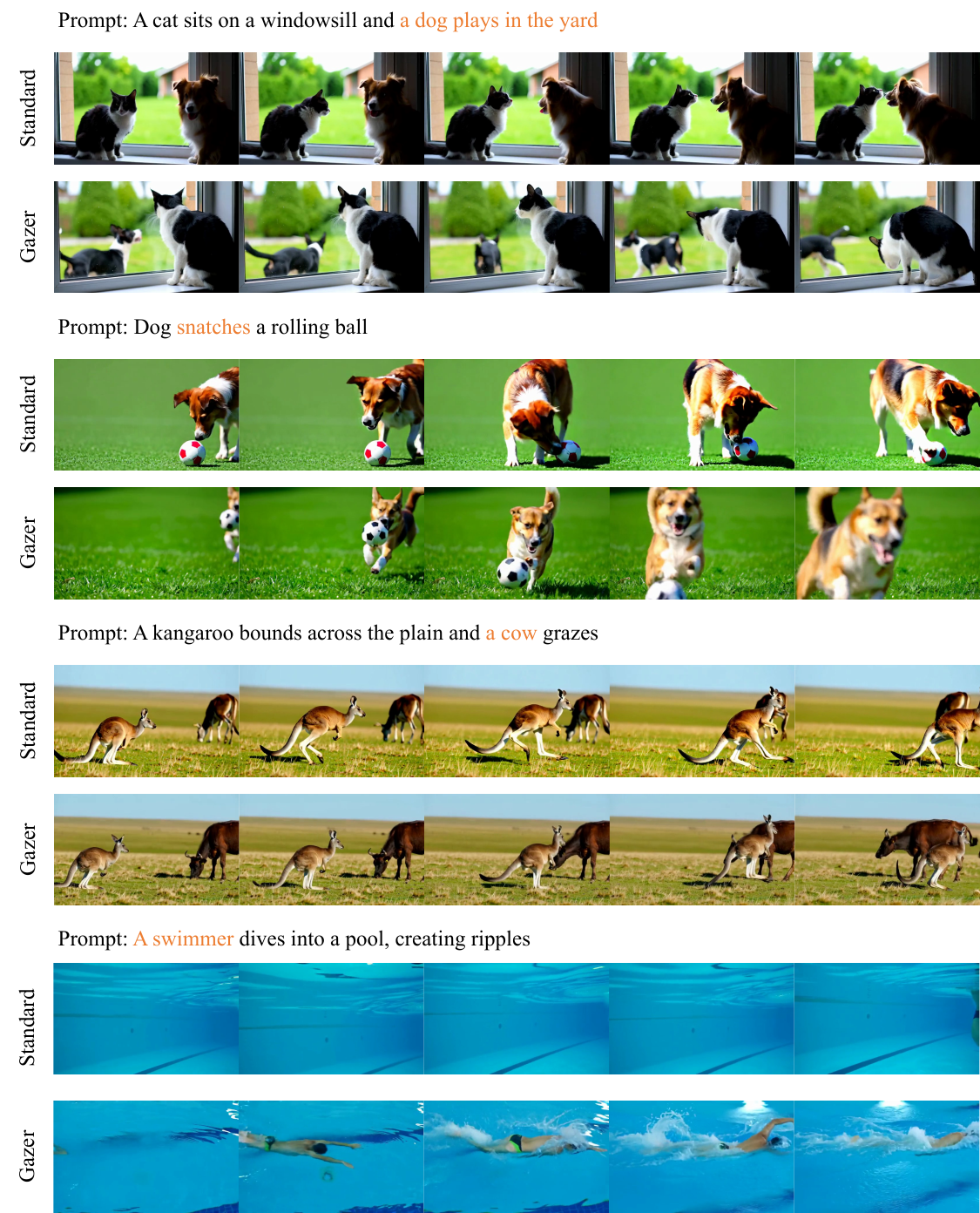}
\caption{
\textbf{Qualitative comparisons with Helios on compositional text-to-video generation.}
}
\label{fig:qual_helios_t2v}
\end{figure*}

We provide additional qualitative comparisons on both text-to-image and text-to-video compositional benchmarks to further illustrate the effectiveness of \ours in improving semantic alignment during hierarchical autoregressive generation.
Figures~\ref{fig:qual_infstar_t2i} and~\ref{fig:qual_star_t2i} present representative results on T2I-CompBench using InfinityStar and STAR, respectively, while Figures~\ref{fig:qual_infstar_t2v} and~\ref{fig:qual_helios_t2v} show text-to-video results on T2V-CompBench using InfinityStar and Helios.

Across diverse prompts involving object composition, attribute binding, spatial relations, and temporal consistency, the baseline models frequently exhibit semantic inconsistencies such as missing objects, incorrect attributes, relation confusion, or progressive drift across scales and frames.
In contrast, \ours produces generations that better preserve the compositional semantics specified in the prompt, yielding improved object completeness, attribute correctness, relational fidelity, and overall semantic coherence.

Notably, the improvements are consistently observed across different autoregressive backbones and generation modalities, demonstrating that the proposed reflective diagnosis and semantic correction mechanism generalizes effectively to both image and video generation settings.
These examples further support our claim that semantic inconsistencies in hierarchical autoregressive generation often emerge progressively along the generation trajectory, and can be mitigated through iterative rollout-based diagnosis and correction.

\section{Prompt Templates for MLLM Feedback}
\label{sec:prompt_templates}

\begin{figure*}[t]
\centering
\includegraphics[width=\textwidth]{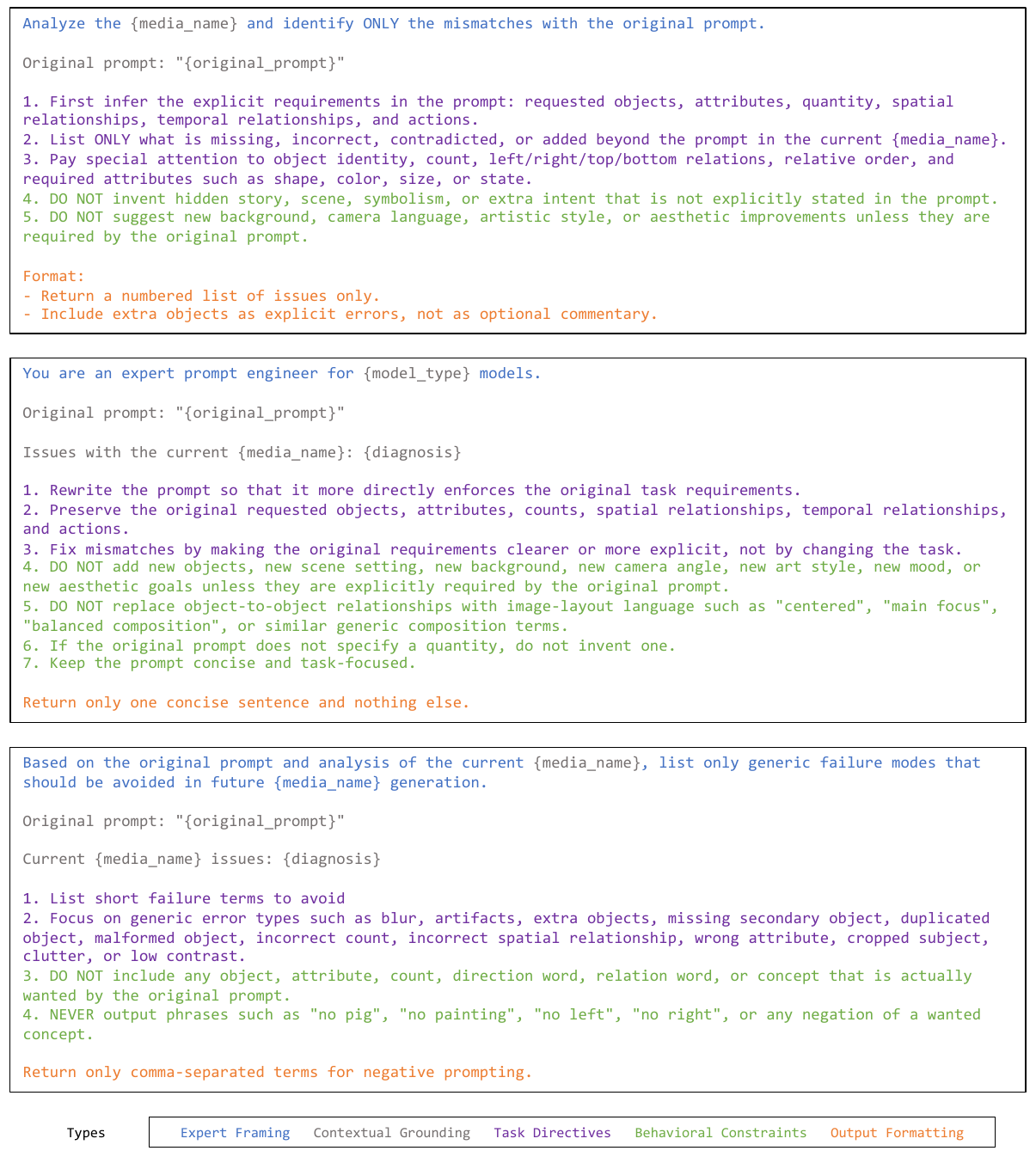}
\caption{\textbf{Prompt templates for MLLM feedback elicitation.} Three structured templates correspond to the three feedback signals produced by \stageone: \textit{Diagnosis}, \textit{Enhancement Cue}, and \textit{Suppression Cue}. Each template is organized into five functional components: Expert Framing, Contextual Grounding, Reasoning Instructions, Behavioral Constraints, and Output Specification.}
\label{fig:prompt_templates}
\end{figure*}

\ours elicits the MLLM using three structured prompt templates corresponding to the three feedback signals produced by \stageone: \textit{Diagnosis}, \textit{Enhancement Cue}, and \textit{Suppression Cue}.
As shown in Figure~\ref{fig:prompt_templates}, each template is organized into five functional components: \textit{Expert Framing} defines the role and perspective of the MLLM; \textit{Contextual Grounding} provides the input prompt and rollout preview; \textit{Reasoning Instructions} guide the MLLM through the diagnostic or rewriting task; \textit{Behavioral Constraints} restrict undesired outputs such as hallucinated objects or aesthetic suggestions; and \textit{Output Specification} enforces a structured, generation-ready format.
This modular design ensures consistent and controllable MLLM behavior across varied inputs and strengthens the reliability of in-generation semantic correction.

%% file: tables/efficiency.tex

\begin{table*}[t]
\centering
\small
\setlength{\tabcolsep}{3pt}
\resizebox{\textwidth}{!}{%
\begin{tabular}{lccccccccc}
\toprule
& \multicolumn{3}{c}{Inference Cost}
& \multicolumn{2}{c}{Memory Usage}
& \multicolumn{4}{c}{Generation Quality} \\
\cmidrule(lr){2-4}
\cmidrule(lr){5-6}
\cmidrule(lr){7-10}

Method
& Avg Time (s)
& Relative Time
& Gen Peak (GB)
& MLLM Peak (GB)
& Combined Peak (GB)
& Consist.\ Attr.
& CLIP Score
& PickScore
& ImageReward \\
\midrule

Baseline (seed 42)
& 100.09
& $1.00\times$
& 59.40
& 0
& 59.40
& 13.07
& 26.66
& 21.51
& 0.2811 \\

Baseline (seed 43)
& 100.09
& $1.00\times$
& 59.40
& 0
& 59.40
& 12.52
& 26.61
& 21.62
& 0.2690 \\

Baseline Best-of-2
& 200.18
& $2.00\times$
& 59.40
& 0
& 59.40
& 13.23
& 26.62
& 21.43
& 0.2802 \\

\ours
& 163.57
& $\mathbf{1.63\times}$
& 59.64
& 17.80
& 77.44
& \textbf{14.20}
& \textbf{27.77}
& \textbf{22.28}
& \textbf{1.1452} \\

\bottomrule
\end{tabular}%
}
\caption{
\textbf{Detailed inference cost and efficiency comparison.}
\ours introduces additional overhead due to rollout preview
and MLLM-based trajectory diagnosis,
but still operates below the cost of Best-of-2 sampling.
The generator-side peak memory remains nearly unchanged,
while the additional memory mainly comes from the MLLM.
Despite the lower inference budget,
\ours consistently achieves superior compositional and preference-alignment scores.
}
\label{tab:efficiency}
\end{table*}

\begin{table*}[t]
\centering
\small
\resizebox{\textwidth}{!}{%
\begin{tabular}{lcccc}
\toprule
Module & Avg Time (s) & Share of Full Pipeline & Relative to Baseline Total & Avg Calls per Video \\
\midrule
Preview & 12.86 & 7.86\% & 12.85\% & 4 \\
MLLM Evaluation & 44.91 & 27.46\% & 44.87\% & 4 \\
Reflection & 1.87 & 1.14\% & 1.87\% & 4 \\
Remaining Generation & 103.93 & 63.54\% & 103.83\% & -- \\
\bottomrule
\end{tabular}%
}
\caption{\textbf{Module-level runtime breakdown of the full \ours pipeline.} The percentage of baseline total time is computed relative to standard sampling runtime.}
\label{tab:efficiency_breakdown}
\end{table*}